\def\year{2020}\relax
\documentclass[letterpaper]{article} 
\usepackage{aaai20}  
\usepackage{times}  
\usepackage{helvet} 
\usepackage{courier}  
\usepackage[hyphens]{url}  
\usepackage{graphicx} 
\usepackage{graphicx}  
\usepackage{color}
\usepackage{subfig}
\usepackage{arydshln}
\usepackage{multirow}

\usepackage[ruled,linesnumbered]{algorithm2e}
\usepackage{amsmath, amsthm, amssymb}
\newtheorem{theorem}{Theorem}

\newtheorem{proposition}[theorem]{Proposition}

\title{Embedding Compression with Isotropic Iterative Quantization}
\author{
Siyu Liao\textsuperscript{1},
Jie Chen\textsuperscript{2},
Yanzhi Wang\textsuperscript{3},
Qinru Qiu\textsuperscript{4},
Bo Yuan\textsuperscript{1}
\\
textsuperscript{1}{Department of Electrical and Computer Engineering, Rutgers University}\\
\textsuperscript{2}{MIT-IBM Watson AI Lab, IBM Research}\\
\textsuperscript{3}{Department of Electrical and Computer Engineering, Northeastern University}\\
\textsuperscript{4}{Electrical Engineering and Computer Science Department, Syracuse University}\\\
siyu.liao@rutgers.edu,
chenjie@us.ibm.com,
yanz.wang@northeastern.edu\\
qiqiu@syr.edu,
bo.yuan@soe.rutgers.edu
}
 
\begin{document}

\maketitle

\begin{abstract}
  Continuous representation of words is a standard component in deep learning-based NLP models. 
  However, representing a large vocabulary requires significant memory, which can cause problems, particularly on resource-constrained platforms. 
  Therefore, in this paper
  we propose an isotropic iterative quantization  (IIQ) approach for compressing embedding vectors into binary ones, leveraging the iterative quantization technique well established for image retrieval, while satisfying the desired isotropic property of PMI based models. Experiments with pre-trained embeddings (i.e., GloVe and HDC) demonstrate a more than thirty-fold compression ratio with comparable and 
  sometimes even improved performance
  over the original real-valued embedding vectors.
\end{abstract}

\section{Introduction}

Words are basic units in many natural language processing (NLP) applications, e.g., translation \cite{bahdanau2014neural} and text classification \cite{joulin2016bag}. Understanding words is crucial but can be very challenging. One difficulty lies in the large vocabulary commonly seen in applications. Moreover, their semantic permutations can be numerous, constituting rich expressions at the sentence and paragraph levels.

In statistical language models, word distributions are learned for unigrams, bigrams, and generally n-grams.
A unigram distribution presents the probability for each word. The histogram is already sufficiently complex given a large vocabulary. Then, the complexity of bigram distributions is quadratic in the vocabulary size and that of n-gram ones is exponential. The combinatorial nature  motivates researchers to develop alternative representations which otherwise explode.

Instead of word distributions, continuous representations with floating-point vectors are much more convenient to handle: they are differentiable, and their differences can be used to draw semantic analogy. A variety of algorithms were proposed over the years for learning these word vectors. Two representative ones are Word2Vec \cite{mikolov2013efficient} and GloVe \cite{pennington2014glove}. Word2Vec is a classical algorithm based on either skip grams or a bag of words, both of which are unsupervised and can directly learn word embeddings from a given corpus. GloVe is another embedding learning algorithm, which combines the advantage of a global factorization of the word co-occurrence matrix, as well as that of the local context. Both approaches are effective in many NLP applications, including word analogy and name entity recognition.

Neural networks with word embeddings are frequently used in solving NLP problems, such as sentiment analysis \cite{dos2014deep} and name entity recognition \cite{lample2016neural}. An advantage of word embeddings is that interactions between words may be modeled by using neural network layers (e.g., attention architectures).

Despite the success of these word embeddings, they often constitute a substantial portion of the overall model. For example, the pre-trained Word2Vec \cite{mikolov2013distributed} contains 3M word vectors and the storage is approximately 3GB. This cost becomes a bottleneck in deployment on resource-constrained platforms.

Thus, much work studies the compression of word embeddings. \cite{shu2017compressing} propose to represent word vectors by using multiple codebooks trained with Gumbel-softmax. \cite{grzegorczyk2017binary} learn binary document emebddings via a bag-of-word-like process. The learned vectors are demonstrated to be effective for document retrieval.

In information retrieval, iterative quantization (ITQ) \cite{gong2013iterative} transforms vectors into binary ones, which are found to be successful in image retrieval. The method maximizes the bit variance meanwhile minimizing the quantization loss. It is theoretically sound and also computationally efficient. However, \cite{grzegorczyk2017binary} find that directly applying ITQ in NLP tasks may not be effective.

In \cite{mu2017all}, authors propose an alternate approach 
that improves the quality of word embeddings without incurring extra training. 
The main idea lies in the concept of isotropy used to explain the success of pointwise mutual information (PMI) based embeddings. The authors demonstrate that the isotropy could be improved through projecting embedding vectors toward weak directions.

Therefore, in this work we propose \emph{isotropic iterative quantization} (IIQ), which leverages iterative quantization meanwhile satisfying the isotropic property. The main idea is to optimize a new objective function regarding the isotropy of word embeddings, rather than maximizing the bit variance. 

Maximizing the bit variance and maximizing isotropy are two opposite ideas, because the former performs projection toward large eigenvalues (dominant directions) while the latter projects toward the smallest ones (weak directions). Given prior success \cite{mu2017all}, it is argued that maximizing isotropy is more beneficial in NLP applications.

\section{Related Work}\label{sec:relatedwork}
In information retrieval (where the proposed method is inspired), locality-sensitive hashing (LSH) is well studied and explored. The aim of LSH is to preserve the similarity between inputs after hashing. This aim is well aligned with that of embedding compression. For example, word similarity can be measured by the cosine distance of their  embeddings. If LSH is applied, the hashed emebddings should maintain a similar distance as the original cosine distance but have much lower complexity in the meantime.

A well-known LSH method in image retrieval is ITQ \cite{gong2013iterative}. However, its application in NLP tasks such as document retrieval is not as successful \cite{grzegorczyk2017binary}. Rather, the authors propose to learn binary paragraph embeddings via a bag-of-words-like model, which essentially computes a binary hash function for the real-valued embedding vectors.

On the other hand, \cite{shu2017compressing} propose a compact structure for embeddings by using the gumble softmax. In this approach, each word vector is represented as the summation of a set of real-valued embeddings. This idea amounts to learning a low-rank representation of the embedding matrix.

Pre-trained embeddings may be directly used in deep neural networks (DNN) or serve as initialization \cite{kim2014convolutional}. There exist several compression techniques for DNNs, including pruning \cite{han2015deep} and low-rank compression \cite{sainath2013low}. Most of these techniques requires retraining for specific tasks, thus challenges exist when applying them to unsupervised word embeddings (e.g., GloVe).

\cite{see2016compression} successfully apply DNN compression techniques to unsupervised embeddings. The authors use pruning to sparsify embedding vectors, which however requires retraining after each pruning iteration. Although retraining is common when compressing DNNs, it often takes a long time to recover the model performance. 
Similarly, \cite{acharya2019online} uses low rank approximation to compress word embeddings, but they also face the same problem to fine-tune a supervised model.

\section{Preliminaries}\label{sec:preliminaries}
\subsection{Iterative Quantization}
In this section, we briefly revisit the iterative quantization method by breaking it down into two steps. The first step is to maximize bit variance when transforming given vectors into binary representation. The second step is about minimizing the quantization loss while maintaining the maximum bit variance. 
\subsubsection{Maximize Bit Variance.}
Let $\mathbf{X}\in\mathbb{R}^{n\times d}$ be the embedding dictionary, where each row $\mathbf{x}^T_i\in\mathbb{R}^d$ denotes the embedding vector for the $i$-th word in the dictionary. Assuming  that vectors are zero centered ($\sum_{i=1}^n \mathbf{x}_i = \mathbf{0}$), ITQ encodes vectors with a binary representation $\{-1,+1\}$ through maximizing the bit variance, which is achieved by solving the following optimization problem: 
\begin{equation}
\label{eqn:max}
\begin{split}
    \max_{\mathbf{W}} \,\, F(\mathbf{W}) &= \frac{1}{n}\text{tr}(
    \mathbf{W}^T\mathbf{X}^T\mathbf{X}\mathbf{W}), \\
    \text{s.t.} \,\, 
    \mathbf{W}^T\mathbf{W}&=\mathbf{I} \,\, \text{and}\,\,  \mathbf{B}=\text{sgn}(\mathbf{X}\mathbf{W}),
\end{split}
\end{equation}
where $\mathbf{W}\in\mathbb{R}^{d\times c}$ and $c\leq d$ is the dimension of the encoded vectors. Here, $\mathbf{B}$ is the final binary representation of $\mathbf{X}$ and $\text{tr}(\cdot)$ and  $\text{sgn}(\cdot)$ are the trace and the sign function, respectively. The problem is the same as that of Principal Component Analysis (PCA) and could be solved by selecting the top $c$ right singular vectors of $\mathbf{X}$ as $\mathbf{W}$. 

\subsubsection{Minimize Quantization Loss.}
Given a solution $\mathbf{W}$ to Equation \eqref{eqn:max}, $\mathbf{U}=\mathbf{W}\mathbf{R}$ is also a solution for any orthogonal matrix $\mathbf{R}\in\mathbb{R}^{c\times c}$. Thus, we could minimize the quantization loss via adjusting the matrix $\mathbf{R}$ while maintaining the solution to \eqref{eqn:max}. The quantization loss is defined as the difference between the vectors before and after the quantization:
\begin{equation}
\label{eqn:quan}
    Q(\mathbf{B}, \mathbf{R}) = || \mathbf{B} - \mathbf{X}\mathbf{W}\mathbf{R} ||_F^2,
\end{equation}
where $||\cdot||_F$ is the Frobenius norm. Note that $\mathbf{B}$ must be binary. The proposed solution in ITQ is an iterative procedure that updates $\mathbf{B}$ and $\mathbf{R}$ in an alternating fashion until convergence. In practice, ITQ turns out able to achieve good performance with early stopping \cite{gong2013iterative}.

\subsection{Isotropy of Word Embedding}
In \cite{arora2016latent}, isotropy is used to explain the success of PMI based word embedding algorithm, for example GloVe embedding. However, \cite{mu2017all} find that existing word embeddings are not nearly isotropic but could be improved. The proposed solution is to project word embeddings toward the weak directions rather than the dominant directions, which seems counter-intuitive but in practice works well. 
The isotropy of word embedding $\mathbf{X}$ is defined as:
\begin{equation}
    I(\mathbf{X})=\frac{\min_{||\mathbf{e}||=1}Z(\mathbf{e})}{\max_{||\mathbf{e}||=1}Z(\mathbf{e})},
\end{equation}
where $Z(\cdot)$ is the partition function
\begin{equation}
    Z(\mathbf{e})=\sum_{\mathbf{x}_i\in\mathbf{X}}\exp(\mathbf{\mathbf{e}}^T\mathbf{x}_i). 
\end{equation}
The value of $I(\mathbf{X})\in [0,1]$ is a measure of isotropy of the given embedding $\mathbf{X}$. A higher $I(\cdot)$ means more isotropic and a better quality of the embedding. It is found making the singular values close to each other can effectively improve embedding isotropy.

\section{Proposed Method}\label{sec:method}
The preceding section hints that maximizing the isotropy and maximizing the bit variance are opposite in action: The former intends to make the singular values close by removing the largest singular values, whereas the latter removes the smallest singular values and maintains the largest. Given the success of isotropy in NLP applications, we propose to minimize the quantization loss while improving the isotropy, rather than maximizing the bit variance. We call the proposed method \emph{isotropic iterative quantization}, IIQ.

The key idea of ITQ is based on the observation that $\mathbf{U}=\mathbf{W}\mathbf{R}$ is still a solution to the objective function of~\eqref{eqn:max}. In our approach IIQ, we show that the orthogonal transformation maintains the isotropy of the input embedding, so that we could apply a similar alternating procedure as in ITQ to minimize the quantization loss. As a result, our method is composed of three steps: maximizing isotropy, reducing dimension, and minimizing quantization loss. 

\subsubsection{Maximize Isotropy.}
The isotropy measure $I(\mathbf{X})$ can be approximated as following \cite{mu2017all} : 
\begin{equation}
\label{eqn:approx}
\hat{I}(\mathbf{X})=
\frac{|\mathbf{X}|-||\mathbf{1}^T\mathbf{X}||+\frac{1}{2}\sigma^2_{\min}}{|\mathbf{X}|+||\mathbf{1}^T\mathbf{X}||+\frac{1}{2}\sigma^2_{\max}},
\end{equation}
where $\sigma_{\min}$ and $\sigma_{\max}$ are the smallest and largest singular values of $\mathbf{X}$, respectively. For $\hat{I}(\mathbf{X})$ to be $1$, the middle term $||\mathbf{1}^T\mathbf{X}||$ on both the numerator and the denominator must be zero and additionally $\sigma_{\min}=\sigma_{\max}$. 
The former requirement can be easily satisfied by the zero-centering given embeddings:
\begin{equation}
\begin{aligned}
\mathbf{u} &= \frac{1}{n}\mathbf{1}^T\cdot\mathbf{X}\\
\mathbf{\bar{X}} &= \mathbf{X} - \mathbf{1}\cdot\mathbf{u}^T,
\end{aligned}
\end{equation}
where $||\mathbf{1}^T\mathbf{\bar{X}}||=0$. The latter may be approximately achieved by removing the large singular values such that the rest of the singular values are close to each other.
A reason why removing the large singular values makes the rest close,  is that often the large singular values have substantial gaps while the rest are clustered. However, removing singular components does not change its dimension.
We denote the maximized result as $\mathbf{\hat{X}}$.

\subsubsection{Dimension Reduction.}
 To make our method more flexible, we perform a dimension reduction afterward by using PCA. This step essentially removes the smallest singular values so that the clustering of the singular values may be further tightened. Note that PCA won't affect the maximized  isotropy of given embeddings, since it only works on the singular values that are already closed to each other after previous step. One can treat the dimension as a hyperparameter, tailored for each data set.

\subsubsection{Minimize Quantization Loss.}

Given a solution $\mathbf{\hat{X}}$ to the maximization of \eqref{eqn:approx}, we prove that multiplying $\mathbf{\hat{X}}$ with an orthogonal matrix $\mathbf{R}$ results in the same $\hat{I}(\mathbf{X})$. In other words, we could minimize the quantization loss \eqref{eqn:quan} while maintaining the isotropy.

\begin{proposition}
If $\mathbf{\hat{X}} \in \mathbb{R}^{n\times d}$ is isotropic and $\mathbf{R}\in\mathbb{R}^{d\times d}$ is orthogonal, then $\mathbf{U}=\mathbf{\hat{X}}\mathbf{R}$ admits $\hat{I}(\mathbf{U})=\hat{I}(\mathbf{\hat{X}})$. 
\end{proposition}

\begin{proof}
Given that $\mathbf{R}$ is orthogonal, we first prove that $\mathbf{U}$ has the same singular values as does $\mathbf{\hat{X}}$. Let $\mathbf{\hat{X}}$ have the singular value decomposition (SVD)
\begin{equation}
    \mathbf{\hat{X}}=\mathbf{P}\text{diag}(\sigma_{\max}, \dots, \sigma_{\min})\mathbf{Q},
\end{equation}
where $\mathbf{P}\in\mathbb{R}^{n\times d}$ and orthogonal matrix $\mathbf{Q}\in\mathbb{R}^{d\times d}$. Let $\mathbf{Q}'=\mathbf{Q}\mathbf{R}$. Then, we have  
\begin{equation}
\label{eqn:svd}
    \mathbf{U}=\mathbf{P}\text{diag}(\sigma_{\max}, \dots, \sigma_{\min})\mathbf{Q}'.
\end{equation}
Since $\mathbf{Q}'$ is also orthogonal, Equation \eqref{eqn:svd} gives the SVD of $\mathbf{U}$. Therefore, $\mathbf{U}$ has the same singular values as does $\mathbf{\hat{X}}$. 

Moreover, $||\mathbf{1}^T\mathbf{U}||=||\mathbf{1}^T\mathbf{\hat{X}}\mathbf{R}||=0$, thus $\mathbf{U}$ is also zero-centered. By Equation \eqref{eqn:approx}, we conclude  $\hat{I}(\mathbf{U})=\hat{I}(\mathbf{\hat{X}})$. 
\end{proof}

With the given proof, we can always use an orthogonal matrix $\mathbf{R}$ to reduce the quantization loss. The iterative optimization strategy as in ITQ \cite{gong2013iterative} is adopted to minimize the quantization loss. Two alternating steps lead to a local minimum. First, compute $\mathbf{B}$ given $\mathbf{R}$:
\begin{equation}
    \mathbf{B} = \text{sgn}(\mathbf{\hat{X}}\cdot \mathbf{R}).
\end{equation}
Second, update $\mathbf{R}$ given $\mathbf{B}$. The update minimizes the quantization loss, which essentially solves the orthogonal Procrustes problem. The solution is given by
\begin{align}
\begin{split}
\mathbf{S}\cdot\mathbf{\Omega}\cdot\mathbf{\hat{S}}^T &= \text{SVD}(\mathbf{B}^T\cdot \mathbf{\hat{X}})\\
\mathbf{R} &= \mathbf{\hat{S}}\cdot\mathbf{S}^T,
\end{split}
\end{align}
where SVD($\cdot$) is the singular value decomposition function and $\mathbf{\Omega}$ is the diagonal matrix of singular values. 
\begin{figure}
    \begin{center}
    \includegraphics[width =\columnwidth]{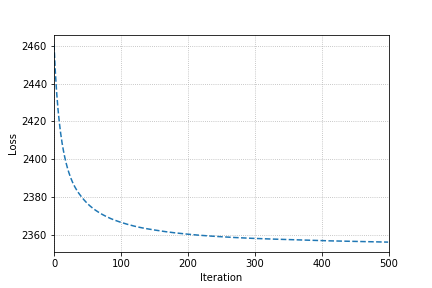}
    \end{center}
    \caption{Quantization loss curve of 50000 embedding vectors from a pre-trained CNN model.}
    \label{fig:loss}
\end{figure}

This iterative updating strategy runs until a local optimal solution is found. Fig. \ref{fig:loss} shows an example of the quantization loss curve. This result is similar to the behavior of ITQ, the authors of which proposed using early stopping to terminate iteration in practice. We follow the guidance and run only 50 iterations in our experiments.

\subsubsection{Overall Algorithm.}
Our method is an unsupervised approach, which does not require any label supervision. Therefore, it can be applied independently of downstream tasks and no fine tuning is needed. This advantage benefits many problems where embeddings often slow down the learning process because of the high space and computation complexity.

We present the pseudocode of the proposed IIQ method in Algorithm~\ref{alg:forward}. The input $D$ denotes the number of top singular values to be removed, $T$ denotes the number of iterations for minimizing the quantization loss, and $O$ denotes the dimension of the output  binary vectors.

The first two lines make zero-centered embedding. Lines 3 to 5 maximize the isotropy. Lines 6 to 8 reduce the embedding dimension, if necessary. Lines 9 to 15 minimize the quantization loss. Within the iteration loop, lines 11 to 12 update $\mathbf{B}$ based on the most recent $\mathbf{R}$, whereas lines 13 to 14 update  $\mathbf{R}$ given the updated $\mathbf{B}$. The last line uses the final transformation $\mathbf{R}$ to return the binary embeddings as output. 
{
\begin{algorithm}[!htbp]
\Indmm
 \KwIn{$\mathbf{X}\in\mathbb{R}^{n\times d}, D, T, O$}
 \KwOut{$\mathbf{B}$}
\Indpp
 $\mathbf{u}\gets \frac{1}{n}\mathbf{1}^T\cdot\mathbf{X}$\;
 $\mathbf{\bar{X}} \gets \mathbf{X} - \mathbf{1}\cdot\mathbf{u}^T$\;

 $\mathbf{S}\cdot\mathbf{\Omega}\cdot\mathbf{\hat{S}}^T\gets \text{SVD}(\mathbf{\bar{X}})$\;
 Set top $D$ singular values in $\mathbf{\Omega}$ as 0\;
 $\mathbf{\hat{X}}\gets \mathbf{S}\cdot\mathbf{\Omega}\cdot\mathbf{\hat{S}}^T$\;

 \If {$O < d$} {
  $\mathbf{\hat{X}}\gets \text{PCA}(\mathbf{\hat{X}}, O)$
 }

 Randomly initialize an orthogonal matrix $\mathbf{R}$\;
 \For{$i \gets 1$ to $T$}{
    $\mathbf{U} \gets \mathbf{\hat{X}}\cdot\mathbf{R}$\;
    $\mathbf{B} \gets \text{sgn}(\mathbf{U})$\;
    $\mathbf{S}\cdot\mathbf{\Omega}\cdot\mathbf{\hat{S}}^T\gets \text{SVD}(\mathbf{B}^T\cdot \mathbf{\hat{X}})$\;
    $\mathbf{R}\gets \mathbf{\hat{S}}\cdot\mathbf{S}^T$\;
}
 \Return{$\text{sgn}(\mathbf{\hat{X}}\cdot\mathbf{R})$}\;
 \caption{Isotropic Iterative Quantization}\label{alg:forward}
\end{algorithm}
}

\section{Experimental Results}\label{sec:experiment}
We run the proposed method on pre-trained embedding vectors and evaluate the compressed embedding in various NLP tasks.
For some tasks, the evaluation is directly conducted over the embedding (e.g., measuring the cosine similarity between word vectors); whereas for others, a classifier is trained with the embedding.
We conduct all experiments in Python by using Numpy and Keras. 
The environment is Ubuntu 16.04 with Intel(R) Xeon(R) CPU E5-2698.

\subsubsection{Pre-trained Embedding.} 
We perform experiments with the GloVe embedding \cite{pennington2014glove} and the HDC embedding \cite{sun2015learning}. 
The GloVe embedding is trained from 42B tokens of Common Crawl data.
The HDC embedding is trained from public Wikipedia. It has a better quality than GloVe because the training process considers both syntagmatic and paradigmatic relations. 
All embedding vectors are used in the experiment without  vocabulary truncation or post-processing.

In addition, we evaluate embedding compression on a CNN model pre-trained with the IMDB data set. 
Different from the prior case, the embedding from CNN is trained with supervised class labels. 
We compress the embedding and retrain the model  to evaluate performance. 
This way enables us to compare with other compression methods fairly.

\begin{table}[t]
\centering
\caption{Experiment Configurations.}
\label{tbl:conf}
\begin{tabular}{|l|l|c|c|c|}
\hline
& \bf Method  & \bf Dimension  & \bf Comp. Ratio \\ \hline
\multirow{7}{*}{\rotatebox[origin=c]{90}{GloVe}} 
& Baseline & $1917494\times 300$ &  1 \\ \cline{2-5} 
& Prune    & $1917494\times 300$ &  20\\ \cline{2-5} 
& DCCL & $M=32, K=256$ & 32 \\ \cline{2-5} 
& NLB & $1917494\times 300$ & 32 \\\cline{2-5}
& ITQ & $1917494\times 300$  &	32 \\ \cline{2-5} 
& \it IIQ-32 & $1917494\times 300$  &	32 \\ \cline{2-5} 
& \it IIQ-64 & $1917494\times 150$ &	64\\ \cline{2-5} 
& \it IIQ-128 &	$1917494\times 75$ &	 128\\ \hline
\multirow{7}{*}{\rotatebox[origin=c]{90}{HDC}} 
& Baseline & $388723\times 300$ &	1\\\cline{2-5} 
& Prune & $388723\times 300$ &		20\\\cline{2-5} 
& DCCL & $M=32, K=128$   & 29 \\\cline{2-5} 
& NLB & $388723\times 300$  & 32 \\\cline{2-5}
& ITQ & $388723\times 300$ &	32\\\cline{2-5} 
& \it IIQ-32 &	$388723\times 300$ &	32\\\cline{2-5} 
& \it IIQ-64 &	$388723\times 150$ &	64\\\cline{2-5} 
& \it IIQ-128 & $388723\times 75$ &	128\\\hline
\end{tabular}
\end{table}

\subsubsection{Configuration.}
We compare IIQ with the traditional ITQ method \cite{gong2013iterative}, the pruning method \cite{see2016compression}, deep compositional code learning (DCCL) \cite{iclrShuN18} and a recent method \cite{tissier2019near} we name as NLB. 
The pruning method is set to prune 95\% of the words for a similar compression ratio. 
The DCCL method is similarly configured.
We run NLB with its default setting. 
We train the DCCL method for 200 epochs and set the batch size to be 1024 for GloVe and 64 for HDC. 
For our method, we set the iteration number $T$ to be 50 since early stopping works sufficiently well. 
We set the same iteration number for ITQ. 
We also set the parameter $D$ to be 2 for HDC, and 14 for Glove embedding. 
Note that we perform all vector operations in real domain on the platform \cite{jastrzebski2017evaluate} and \cite{conneau2018senteval}.

\begin{table*}[]
\centering
\caption{Word Similarity Results.}
\label{tbl:sim}
\begin{tabular}{|l|l|c|c|c|c|c|c|c|}
\hline
&  \bf Method  & \bf MEN &\bf	MTurk&\bf	RG65&\bf	RW&\bf	SimLex999&\bf	TR9856&\bf	WS353 \\ \hline\hline
\multirow{7}{*}{\rotatebox[origin=c]{90}{GloVe}} 
&Baseline & 73.62 & 64.50 & 81.71 & 37.43 & 37.38 & 9.67 & 69.07 \\ \cdashline{2-9} 
&Prune & 17.97 & 22.09 & 39.66 & 12.45 & -0.37 & 8.31 & 14.52 \\ \cdashline{2-9} 
&DCCL & 54.46 & 50.46 & 63.89 & 28.04 & 25.48 & 7.91 & 54.55 \\ \cdashline{2-9}
&NLB & 73.99 & \bf 64.98 & 72.07 & 40.86 & 40.52 & \bf 14.00 & 66.09 \\ \cdashline{2-9}
&ITQ & 57.37 & 52.93 & 72.08 & 25.10 & 26.23 & 8.98 & 55.00 \\ \cdashline{2-9} 
&\it IIQ-32 & \bf 76.43 & 63.33 & \bf 78.16 & \bf 41.35 & \bf 41.87 & 9.80 & \bf 72.22 \\
&\it IIQ-64 & 71.55&58.37&74.94&37.61&38.80&12.81&67.99\\
&\it IIQ-128&59.25 &50.42&62.39&28.71&33.25&12.31&53.56\\
\hline\hline
\multirow{7}{*}{\rotatebox[origin=c]{90}{HDC}} 
&Baseline & 76.03 & 65.77 & 80.58 & 46.34 & 40.68 & 20.71 & 76.81 \\ \cdashline{2-9} 
&Prune & 46.83 & 41.49 & 56.14 & 29.84 & 26.27 & 15.27 & 52.06\\ \cdashline{2-9}
&DCCL & 68.82 & 55.78 & 72.23 & \bf 39.33 & 35.02 & \bf 18.41 & 66.09 \\\cdashline{2-9} 
&NLB & 72.06 & 61.57 & 72.58 & 35.45 & 38.50 & 11.71 & 67.20 \\ \cdashline{2-9} 
&ITQ & 72.31 & 61.68 & 74.70 & 37.01 & 37.40 & 9.69 & 72.32 \\ \cdashline{2-9} 
&\it IIQ-32 & \bf 74.37 &  \bf 66.71 &  \bf 78.04 & 38.75 & \bf 39.35 & 9.63 & \bf 75.32 \\
&\it IIQ-64 & 66.32 & 56.73 & 65.77 & 35.63 & 36.22 & 11.33 & 72.70 \\ 
&\it IIQ-128 & 55.83 & 51.33 & 45.76 & 32.03 & 29.45 & 12.61 & 58.54 \\
\hline
\end{tabular}
\end{table*}

Table \ref{tbl:conf} lists the experiment configurations with method name, dimension, embedding value type, and compression ratio. 
The baseline means the original embedding. 
Our method starts with ``IIQ,'' followed by the compression ratio. 
The ``dimension'' column gives the number of vectors and the vector dimension. 
For DCCL, we list the parameters $M$ and $K$ that determine the  compression ratio.
Note that we use single precision for real values. 
The last column shows the compression ratio, which is the  the size of the original embedding over that of the compressed one. 
Thus, the compression from real value to binary is 32. 
Moreover, we also apply dimension reduction in IIQ so that higher compression ratio is possible.

\subsection{Word Similarity}

The task measures Spearman's rank correlation between word vector similarity and human rated similarity.
A higher correlation means a better quality of the word embedding. 
The similarity between two words is computed as the cosine of the corresponding vectors, i.e., $\cos(\mathbf{x}, \mathbf{y})=\mathbf{x}^T\mathbf{y}/(||\mathbf{x}||\cdot||\mathbf{y}||)$, where $\mathbf{x}$ and $\mathbf{y}$ are two word vectors. 
Out-of-vocabulary (OOV) words are replaced by the mean vector.

In this experiment, seven data sets are used, including 
MEN \cite{bruni2014multimodal} with 3000 pairs of words obtained from Amazon crowdsourcing; 
MTurk \cite{radinsky2011word} with 287 pairs, focusing on word semantic relatedness; 
RG65 \cite{rubenstein1965contextual} with 65 pairs, an early published dataset;
RW \cite{luong2013better} with 2034 pairs of rare words selected based on frequencies; 
SimLex999 \cite{hill2015simlex} with 999 pairs, aimed at genuine similarity estimation; 
TR9856 \cite{levy2015tr9856} with 9856 pairs, containing many acronyms and name entities; and
WS353 \cite{agirre2009study} with 353 pairs of mostly verbs and nouns. 
The experiment is conducted on the platform \cite{jastrzebski2017evaluate}.

Table \ref{tbl:sim} summarizes the results.
The performance of IIQ degrades as the compression ratio increases.
This is expected, since a higher compression ratio leads to more loss of information.
In addition, our IIQ method consistently achieves better results than ITQ, DCCL, NLB and the pruning method. 
Particularly, one sees that on the Men data set, IIQ even outperforms the baseline embedding Glove. 
Another observation is that on TR9856, a higher compression ratio surprisingly yields better results for IIQ. 
We speculate that the cause is the multi-word term relations unique to TR9856.
Interestingly, the pruning method results in negative correlation in SimLex999 for the GloVe embedding. 
This means that pruning too many small values inside word embedding can drastically destroy the embedding quality.

\subsection{Categorization}

The task is to cluster words into different categories. 
The performance is measured by purity, which is defined as the fraction of correctly classified words.
We run the experiment using agglomerative clustering and k-means clustering, and select the highest purity as the final result for each embedding. 
This experiment is conducted on the platform \cite{jastrzebski2017evaluate} where OOV words are replaced by the mean vector.

Four data sets are used in this experiment: Almuhareb-Poesio (AP)  
\cite{almuhareb2005concept} with 402 words in 21 categories;
BLESS \cite{baroni2011we} with 200 nouns (animate or inanimate) in 17 categories;
Battig \cite{battig1969category} with 5231 words in 56 taxonomic categories; and
ESSLI2008 Workshop \cite{baroni2008} with 45 verbs in 9 semantic categories.

Table \ref{tbl:cat} lists evaluation results for GloVe and HDC embeddings.
One sees that the proposed IIQ method works better than ITQ, DCCL, and the pruning method on all data sets.
But NLB sometimes achieves the best result for example on Battig. 
For ESSLI, IIQ even outperforms the original GloVe and HDC embedding.

\begin{table}[h]
\centering
\caption{Categorization Results.}
\label{tbl:cat}
\begin{tabular}{|l|l|c|c|c|c|}
\hline
& \bf Method &  \bf AP & \bf BLESS & \bf Battig & \bf ESSLI \\ \hline\hline
\multirow{7}{*}{\rotatebox[origin=c]{90}{GloVe}} 
&Baseline & 62.94 & 78.50 & 45.13 & 57.78 \\ \cdashline{2-6} 
&Prune & 38.56 & 46.00 & 23.42 & 42.22 \\\cdashline{2-6}
&DCCL & 52.24 & 75.00 & 36.09 & 48.89 \\\cdashline{2-6} 
&NLB& 59.45 & 78.50 & \bf 43.39 & \bf 66.67 \\ \cdashline{2-6} 
&ITQ & 58.71 & 76.50 & 40.76 & 48.89 \\\cdashline{2-6} 
&\it IIQ-32 & \bf 64.18 & \bf 80.00 & 41.98 & 60.00\\
&\it IIQ-64 & 56.22 & 76.50 & 37.49 & 51.11 \\
&\it IIQ-128 & 45.02 & 69.00 & 31.43 &  44.44\\
\hline\hline
\multirow{7}{*}{\rotatebox[origin=c]{90}{HDC}} 
&Baseline & 65.42 & 81.50 & 43.18 & 60.00 \\ \cdashline{2-6} 
&Prune &  34.33 & 48.00 & 23.28 & 51.11 
\\ \cdashline{2-6} 
&DCCL & 55.97 & 74.50 & 40.16 &  53.33\\ \cdashline{2-6}
&NLB & 59.20 & 75.50 & \bf 41.88 & 62.22 \\ \cdashline{2-6}
&ITQ & 57.21 & 77.50 & 41.04 & 55.56 \\ \cdashline{2-6} 
&\it IIQ-32 & \bf 61.69 & \bf 78.00 & 41.29  & \bf 62.22 \\ 
&\it IIQ-64 & 48.51 & 72.50 & 35.90 & 53.33 \\ 
&\it IIQ-128 & 43.03 & 57.50 & 28.50 & 62.22 \\
\hline
\end{tabular}
\end{table}

\subsection{Topic Classification}
In this experiment, we perform topic classification by using sentence embedding.
The embedding is computed as the average of the corresponding word vectors.
The average of binary embedding is fed to the classifier in single precision. 
Missing words are treated as zero and so are OOV words. 
In this task, we train a Multi-Layer Perceptron (MLP) as the classifier for each method. Due to the different size of embeddings, we train 10 epochs for all Glove embeddings and 4 epochs for all HDC embedding. Five-fold cross validation is used to report classification accuracy. 

Four data sets are selected from \cite{wang2012baselines}, including movie review (MR), customer review (CR), opinion-polarity (MPQA), and subjectivity (SUBJ). Similar performance is achieved by using the original embedding. 
The experiment is conducted on the platform of \cite{conneau2018senteval}. 

Table \ref{tbl:class} shows the results for each method. 
Similar to the previous tasks, the proposed IIQ method consistently performs better than ITQ, pruning, and DCCL. 
The only exception is that for MPQA and SUBJ, DCCL and NLB achieves the best result for the GloVe embedding respectively. 
As the compression ratio increases, IIQ encounters performance degrade.

\begin{table}[!h]
\centering
\caption{Topic Classification Results.}
\label{tbl:class}
\begin{tabular}{|l|l|c|c|c|c|}
\hline 
& \bf Method & \bf CR & \bf MPQA & \bf MR & \bf SUBJ \\ \hline\hline
\multirow{7}{*}{\rotatebox[origin=c]{90}{GloVe}} 
&Baseline & 78.78 & 87.16 & 76.42 & 91.29 \\\cdashline{2-6} 
&Prune & 73.48 & 81.93 & 71.97 & 87.19 \\\cdashline{2-6} 
&DCCL & 77.27 & \bf 85.6 & 74.74 & 89.56 \\\cdashline{2-6} 
&NLB & 75.36 & 85.77 & 73.01 & \bf 89.92 \\\cdashline{2-6} 
&ITQ & 71.79 & 84.11 & 73.18 & 89.55 \\\cdashline{2-6} 
&\it IIQ-32 & \bf 77.7 & 85.15 & \bf 74.96 & 89.87\\
&\it IIQ-64&75.07&83.02&73.17&88.14\\
&\it IIQ-128 & 72.56&80.55&69.93&84.29\\
\hline\hline 
\multirow{7}{*}{\rotatebox[origin=c]{90}{HDC}} 
&Baseline&	76.40&	86.61&	75.71&	90.86\\\cdashline{2-6} 
&Prune	&70.97	&78.84	&67.56	&83.58\\\cdashline{2-6} 
&DCCL&	74.68	&84.2	&73.32	&89.43\\\cdashline{2-6} 
&NLB	&70.89	&84.51	&73.18	&89.48\\\cdashline{2-6} 
&ITQ	&73.57	&84.44	&72.3	&89.46\\\cdashline{2-6} 
&\it IIQ-32&	\bf 76.32&	\bf 84.77&	\bf 73.51&	\bf 89.91\\
&\it IIQ-64&	72.18&	82.07&	70.32&	87.41\\
&\it IIQ-128&	70.83&	77.62&	67.89&	84.62\\
\hline
\end{tabular}
\end{table}

\begin{table}[h]
\begin{center}
\caption{Configurations for IMDB Classification.}
\label{tbl:confimdb}
\begin{tabular}{|l|c|c|c|c|}
\hline \bf Method &  \bf Dimension & \bf Comp. Ratio \\ \hline
Baseline & $50000\times 300$ & 	1\\\hline
Prune & $50000\times 300$ &		20\\\hline
DCCL & $M=32, K=32$ & 27 \\\hline
NLB & $50000\times 300$ & 32 \\\hline
ITQ & $50000\times 300$ & 	32\\\hline
\it IIQ-32&	$50000\times 300$ &		32\\\hline
\it IIQ-64&	$50000\times 150$ &		64\\\hline
\it IIQ-128 & $50000\times 75$ &	128\\\hline
\end{tabular}
\end{center}
\end{table}

\subsection{Sentiment Analysis}

In this experiment, we evaluate over the embedding input to a pre-trained Convolutional Neural Network (CNN) model on the IMDB data set \cite{maas-EtAl:2011:ACL-HLT2011}. 
The CNN model follows the Keras tutorial \cite{kerastutorial}. We train 50,000 embedding vectors in 300 dimensions.
The model is composed of an embedding layer, followed by a dropout layer with probability 0.2, a 1D convolution layer with 250 filters and kernel size 3, a 1D max pooling layer, a fully connected layer with hidden dimension 250, a dropout layer with probability 0.2, a ReLU activation layer, and a single output fully connected layer with sigmoid activation. 
Moreover, we use adam optimizer with learning rate 0.0001, sentence length 400, batch size 128, and train for 20 epochs. 
Input embedding fed into CNN is kept fixed (not trainable).

The data set contains 25,000 movie reviews for training and another 25,000 for testing. 
We randomly separate 5,000 reviews from the training set as validation data.
The model with the best performance on the validation set is kept as the final model for measuring test accuracy.
Moreover, all results are averaged from 10 runs for each embedding. 
The baseline model is the pre-trained CNN model with 87.89\% accuracy. 
Table \ref{tbl:confimdb} summarizes the configurations for this experiment. 
All configurations are similar to the previous experiments. 
The DCCL method is now configured with $M=32$ and $K=32$ to achieve a similar compression ratio.

\begin{figure}[!h]
    \begin{center}
    \includegraphics[width =\columnwidth]{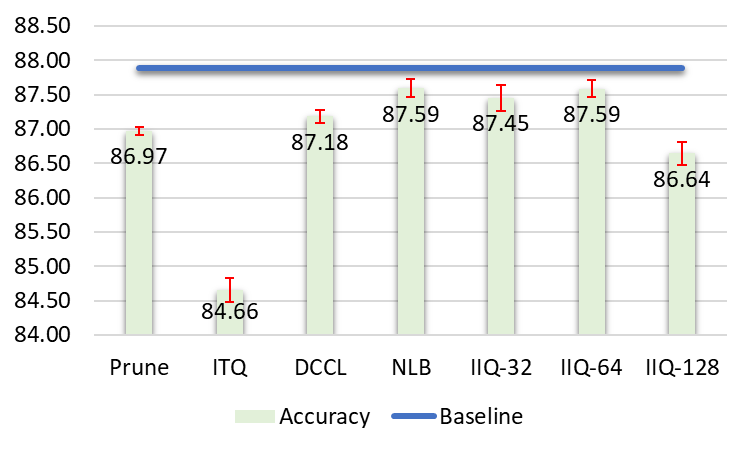}
    \end{center}
    \caption{IMDB CNN Test Accuracy Results.}
    \label{fig:imdb}
\end{figure}

\begin{figure*}[h]%
    \centering
    \subfloat[Nearest and furthest 100 words of ``cook'' in IIQ-GloVe. ]{{\includegraphics[width=\columnwidth]{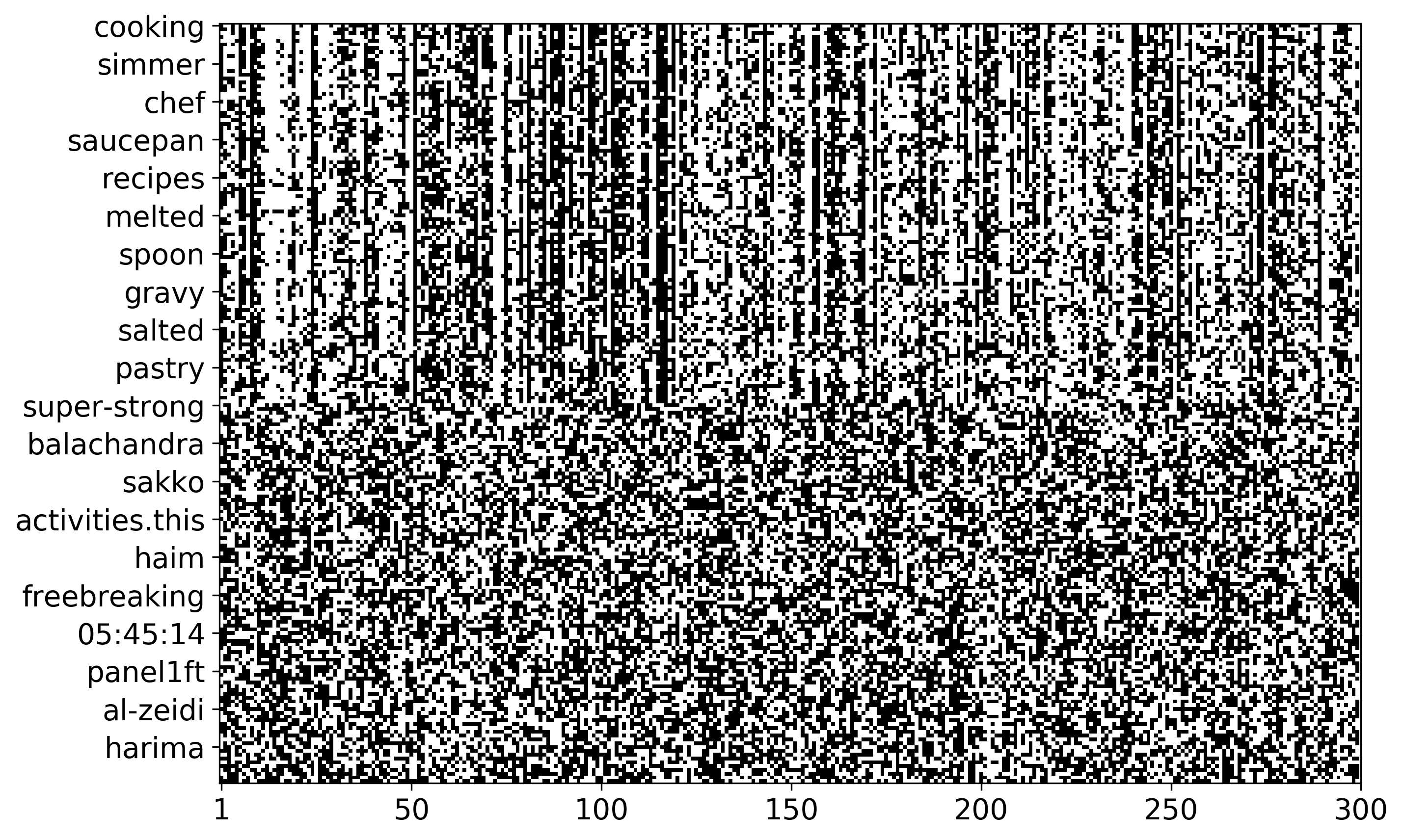} }}%
    \qquad
    \subfloat[Nearest and furthest 100 words of ``man'' in IIQ-HDC.]{{\includegraphics[width=.97\columnwidth]{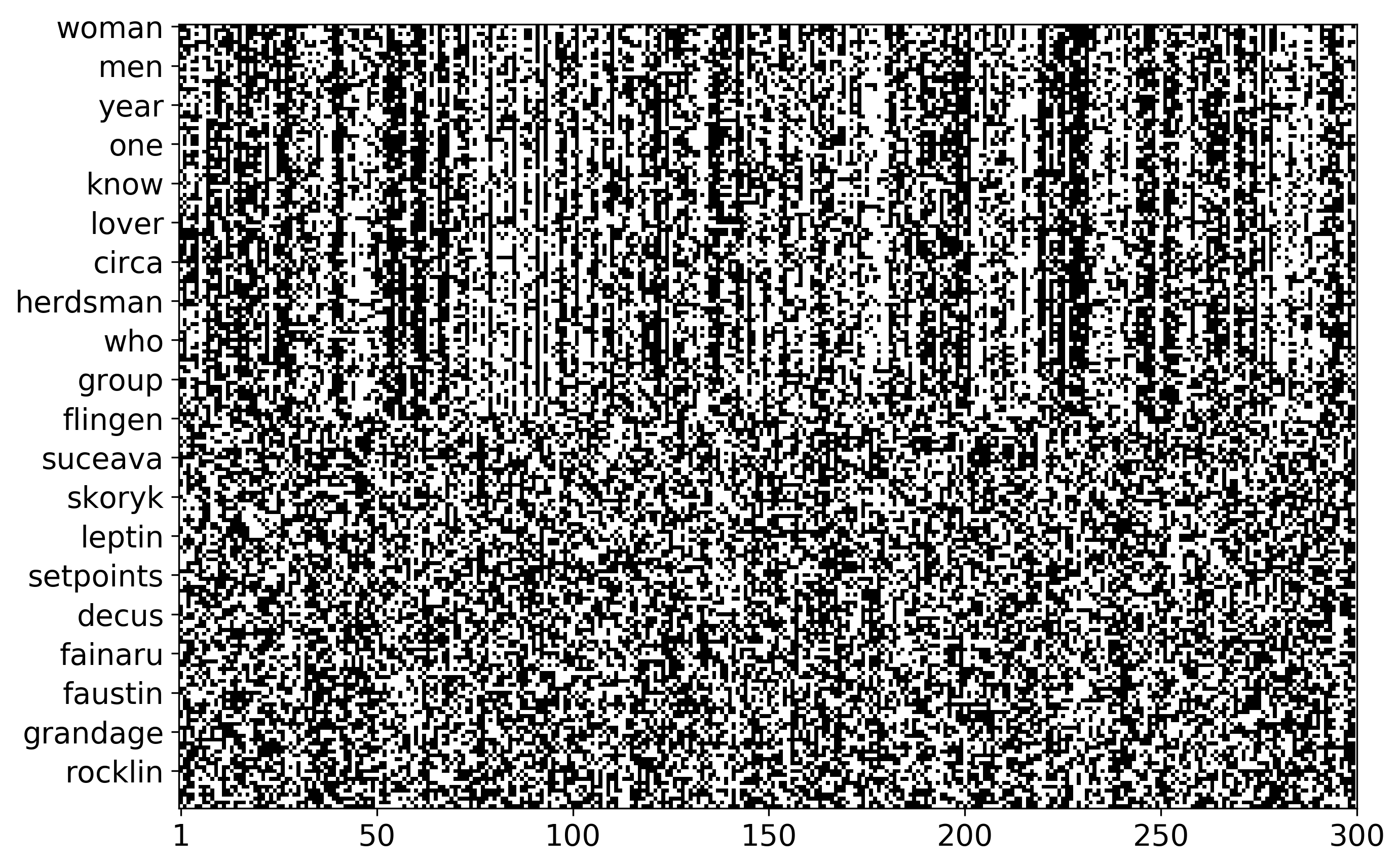} }}%
    \caption{Visualizing Binary IIQ Word Embedding.}
    \label{fig:vis}%
\end{figure*}

We present in Fig. \ref{fig:imdb} the result of each embedding. 
The histogram shows the average accuracy of 10 runs experiments for each method and the error bar shows the standard deviation. 
One sees that among all compression methods, IIQ  achieves the least performance degrade. 
IIQ with compression ratio 64 is the best.

\subsection{Visualization}
We visualize the binary IIQ embedding in Fig. \ref{fig:vis} 
The nearest and furthest 100 word vectors are shown. 
The distance is calculated by the dot product. 
Fig. \ref{fig:vis}(a) shows the IIQ-compressed GloVe embedding and Fig. \ref{fig:vis}(b) shows the IIQ-compressed HDC embedding. 
The y axis lists every 10 words and the x axis is the dimension of the embedding. 
One sees that similar word vectors have similar patterns in many dimensions.
A white column means that the dimension is zero for all words. A black column means one.
Moreover, there is obvious difference between nearest and furthest words.

\section{Conclusion}\label{sec:conclusion}
This paper presents an isotropic iterative quantization (IIQ) method for compressing word embeddings. 
While it is based on the ITQ method in image retrieval, it also maintains the embedding isotropy. 
We evaluate the proposed method on GloVe and HDC embeddings and show that it is effective for word similarity, categorization, and several other downstream tasks. For pre-trained embeddings that are less isotropic (e.g., GloVe), IIQ performs better than ITQ owing to the improvement on isotropy. These findings are based on a 32-fold (and higher) compression ratio. The results point to promising deployment of trained neural network models with word embeddings on resource constrained platforms in real life.

\bibliographystyle{named}
\bibliography{ref}

\begin{thebibliography}{}

\bibitem[\protect\citeauthoryear{Acharya \bgroup \em et al.\egroup
  }{2019}]{acharya2019online}
Anish Acharya, Rahul Goel, Angeliki Metallinou, and Inderjit Dhillon.
\newblock Online embedding compression for text classification using low rank
  matrix factorization.
\newblock In {\em Proceedings of the AAAI Conference on Artificial
  Intelligence}, volume~33, pages 6196--6203, 2019.

\bibitem[\protect\citeauthoryear{Agirre \bgroup \em et al.\egroup
  }{2009}]{agirre2009study}
Eneko Agirre, Enrique Alfonseca, Keith Hall, Jana Kravalova, Marius
  Pa{\c{s}}ca, and Aitor Soroa.
\newblock A study on similarity and relatedness using distributional and
  wordnet-based approaches.
\newblock In {\em Proceedings of Human Language Technologies: The 2009 Annual
  Conference of the North American Chapter of the Association for Computational
  Linguistics}, pages 19--27. Association for Computational Linguistics, 2009.

\bibitem[\protect\citeauthoryear{Almuhareb and
  Poesio}{2005}]{almuhareb2005concept}
Abdulrahman Almuhareb and Massimo Poesio.
\newblock Concept learning and categorization from the web.
\newblock In {\em proceedings of the annual meeting of the Cognitive Science
  society}, 2005.

\bibitem[\protect\citeauthoryear{Arora \bgroup \em et al.\egroup
  }{2016}]{arora2016latent}
Sanjeev Arora, Yuanzhi Li, Yingyu Liang, Tengyu Ma, and Andrej Risteski.
\newblock A latent variable model approach to pmi-based word embeddings.
\newblock {\em Transactions of the Association for Computational Linguistics},
  4:385--399, 2016.

\bibitem[\protect\citeauthoryear{Bahdanau \bgroup \em et al.\egroup
  }{2014}]{bahdanau2014neural}
Dzmitry Bahdanau, Kyunghyun Cho, and Yoshua Bengio.
\newblock Neural machine translation by jointly learning to align and
  translate.
\newblock {\em arXiv preprint arXiv:1409.0473}, 2014.

\bibitem[\protect\citeauthoryear{Baroni and Lenci}{2011}]{baroni2011we}
Marco Baroni and Alessandro Lenci.
\newblock How we blessed distributional semantic evaluation.
\newblock In {\em Proceedings of the GEMS 2011 Workshop on GEometrical Models
  of Natural Language Semantics}, pages 1--10. Association for Computational
  Linguistics, 2011.

\bibitem[\protect\citeauthoryear{Battig and
  Montague}{1969}]{battig1969category}
William~F Battig and William~E Montague.
\newblock Category norms of verbal items in 56 categories a replication and
  extension of the connecticut category norms.
\newblock {\em Journal of experimental Psychology}, 80(3p2):1, 1969.

\bibitem[\protect\citeauthoryear{Bruni \bgroup \em et al.\egroup
  }{2014}]{bruni2014multimodal}
Elia Bruni, Nam-Khanh Tran, and Marco Baroni.
\newblock Multimodal distributional semantics.
\newblock {\em Journal of Artificial Intelligence Research}, 49:1--47, 2014.

\bibitem[\protect\citeauthoryear{Chollet and others}{}]{kerastutorial}
Francois Chollet et~al.
\newblock Keras documentation, convolution1d for text classification.
\newblock \url{https://keras.io/examples/imdb_cnn/}.
\newblock Accessed: 2019-08.

\bibitem[\protect\citeauthoryear{Conneau and Kiela}{2018}]{conneau2018senteval}
Alexis Conneau and Douwe Kiela.
\newblock Senteval: An evaluation toolkit for universal sentence
  representations.
\newblock {\em arXiv preprint arXiv:1803.05449}, 2018.

\bibitem[\protect\citeauthoryear{dos Santos and Gatti}{2014}]{dos2014deep}
Cicero dos Santos and Maira Gatti.
\newblock Deep convolutional neural networks for sentiment analysis of short
  texts.
\newblock In {\em Proceedings of COLING 2014, the 25th International Conference
  on Computational Linguistics: Technical Papers}, pages 69--78, 2014.

\bibitem[\protect\citeauthoryear{Gong \bgroup \em et al.\egroup
  }{2013}]{gong2013iterative}
Yunchao Gong, Svetlana Lazebnik, Albert Gordo, and Florent Perronnin.
\newblock Iterative quantization: A procrustean approach to learning binary
  codes for large-scale image retrieval.
\newblock {\em IEEE Transactions on Pattern Analysis and Machine Intelligence},
  35(12):2916--2929, 2013.

\bibitem[\protect\citeauthoryear{Grzegorczyk and
  Kurdziel}{2017}]{grzegorczyk2017binary}
Karol Grzegorczyk and Marcin Kurdziel.
\newblock Binary paragraph vectors.
\newblock In {\em Proceedings of the 2nd Workshop on Representation Learning
  for NLP}, pages 121--130, 2017.

\bibitem[\protect\citeauthoryear{Han \bgroup \em et al.\egroup
  }{2015}]{han2015deep}
Song Han, Huizi Mao, and William~J Dally.
\newblock Deep compression: Compressing deep neural networks with pruning,
  trained quantization and huffman coding.
\newblock {\em arXiv preprint arXiv:1510.00149}, 2015.

\bibitem[\protect\citeauthoryear{Hill \bgroup \em et al.\egroup
  }{2015}]{hill2015simlex}
Felix Hill, Roi Reichart, and Anna Korhonen.
\newblock Simlex-999: Evaluating semantic models with (genuine) similarity
  estimation.
\newblock {\em Computational Linguistics}, 41(4):665--695, 2015.

\bibitem[\protect\citeauthoryear{Jastrzebski \bgroup \em et al.\egroup
  }{2015}]{jastrzebski2017evaluate}
Stanis{\l}aw Jastrzebski, Damian Le{\'s}niak, and Wojciech~Marian Czarnecki.
\newblock word-embeddings-benchmarks.
\newblock \url{https://github.com/kudkudak/word-embeddings-benchmarks}, 2015.

\bibitem[\protect\citeauthoryear{Joulin \bgroup \em et al.\egroup
  }{2016}]{joulin2016bag}
Armand Joulin, Edouard Grave, Piotr Bojanowski, and Tomas Mikolov.
\newblock Bag of tricks for efficient text classification.
\newblock {\em arXiv preprint arXiv:1607.01759}, 2016.

\bibitem[\protect\citeauthoryear{Kim}{2014}]{kim2014convolutional}
Yoon Kim.
\newblock Convolutional neural networks for sentence classification.
\newblock {\em arXiv preprint arXiv:1408.5882}, 2014.

\bibitem[\protect\citeauthoryear{Lample \bgroup \em et al.\egroup
  }{2016}]{lample2016neural}
Guillaume Lample, Miguel Ballesteros, Sandeep Subramanian, Kazuya Kawakami, and
  Chris Dyer.
\newblock Neural architectures for named entity recognition.
\newblock {\em arXiv preprint arXiv:1603.01360}, 2016.

\bibitem[\protect\citeauthoryear{Levy \bgroup \em et al.\egroup
  }{2015}]{levy2015tr9856}
Ran Levy, Liat Ein-Dor, Shay Hummel, Ruty Rinott, and Noam Slonim.
\newblock Tr9856: A multi-word term relatedness benchmark.
\newblock In {\em Proceedings of the 53rd Annual Meeting of the Association for
  Computational Linguistics and the 7th International Joint Conference on
  Natural Language Processing (Volume 2: Short Papers)}, volume~2, pages
  419--424, 2015.

\bibitem[\protect\citeauthoryear{Luong \bgroup \em et al.\egroup
  }{2013}]{luong2013better}
Thang Luong, Richard Socher, and Christopher Manning.
\newblock Better word representations with recursive neural networks for
  morphology.
\newblock In {\em Proceedings of the Seventeenth Conference on Computational
  Natural Language Learning}, pages 104--113, 2013.

\bibitem[\protect\citeauthoryear{M~Baroni and Lenci}{2008}]{baroni2008}
S~Evert M~Baroni and A~Lenci.
\newblock Bridging the gap between semantic theory and computational
  simulations: Proceedings of the esslli workshop on distributional lexical
  semantics.
\newblock In {\em Proceedings of the esslli workshop on distributional lexical
  semantics}, 2008.

\bibitem[\protect\citeauthoryear{Maas \bgroup \em et al.\egroup
  }{2011}]{maas-EtAl:2011:ACL-HLT2011}
Andrew~L. Maas, Raymond~E. Daly, Peter~T. Pham, Dan Huang, Andrew~Y. Ng, and
  Christopher Potts.
\newblock Learning word vectors for sentiment analysis.
\newblock In {\em Proceedings of the 49th Annual Meeting of the Association for
  Computational Linguistics: Human Language Technologies}, pages 142--150,
  Portland, Oregon, USA, June 2011. Association for Computational Linguistics.

\bibitem[\protect\citeauthoryear{Mikolov \bgroup \em et al.\egroup
  }{2013a}]{mikolov2013efficient}
Tomas Mikolov, Kai Chen, Greg Corrado, and Jeffrey Dean.
\newblock Efficient estimation of word representations in vector space.
\newblock {\em arXiv preprint arXiv:1301.3781}, 2013.

\bibitem[\protect\citeauthoryear{Mikolov \bgroup \em et al.\egroup
  }{2013b}]{mikolov2013distributed}
Tomas Mikolov, Ilya Sutskever, Kai Chen, Greg~S Corrado, and Jeff Dean.
\newblock Distributed representations of words and phrases and their
  compositionality.
\newblock In {\em Advances in neural information processing systems}, pages
  3111--3119, 2013.

\bibitem[\protect\citeauthoryear{Mu \bgroup \em et al.\egroup
  }{2017}]{mu2017all}
Jiaqi Mu, Suma Bhat, and Pramod Viswanath.
\newblock All-but-the-top: Simple and effective postprocessing for word
  representations.
\newblock {\em arXiv preprint arXiv:1702.01417}, 2017.

\bibitem[\protect\citeauthoryear{Pennington \bgroup \em et al.\egroup
  }{2014}]{pennington2014glove}
Jeffrey Pennington, Richard Socher, and Christopher Manning.
\newblock Glove: Global vectors for word representation.
\newblock In {\em Proceedings of the 2014 conference on empirical methods in
  natural language processing (EMNLP)}, pages 1532--1543, 2014.

\bibitem[\protect\citeauthoryear{Radinsky \bgroup \em et al.\egroup
  }{2011}]{radinsky2011word}
Kira Radinsky, Eugene Agichtein, Evgeniy Gabrilovich, and Shaul Markovitch.
\newblock A word at a time: computing word relatedness using temporal semantic
  analysis.
\newblock In {\em Proceedings of the 20th international conference on World
  wide web}, pages 337--346. ACM, 2011.

\bibitem[\protect\citeauthoryear{Rubenstein and
  Goodenough}{1965}]{rubenstein1965contextual}
Herbert Rubenstein and John~B Goodenough.
\newblock Contextual correlates of synonymy.
\newblock {\em Communications of the ACM}, 8(10):627--633, 1965.

\bibitem[\protect\citeauthoryear{Sainath \bgroup \em et al.\egroup
  }{2013}]{sainath2013low}
Tara~N Sainath, Brian Kingsbury, Vikas Sindhwani, Ebru Arisoy, and Bhuvana
  Ramabhadran.
\newblock Low-rank matrix factorization for deep neural network training with
  high-dimensional output targets.
\newblock In {\em 2013 IEEE international conference on acoustics, speech and
  signal processing}, pages 6655--6659. IEEE, 2013.

\bibitem[\protect\citeauthoryear{See \bgroup \em et al.\egroup
  }{2016}]{see2016compression}
Abigail See, Minh-Thang Luong, and Christopher~D Manning.
\newblock Compression of neural machine translation models via pruning.
\newblock {\em arXiv preprint arXiv:1606.09274}, 2016.

\bibitem[\protect\citeauthoryear{Shu and Nakayama}{2017}]{shu2017compressing}
Raphael Shu and Hideki Nakayama.
\newblock Compressing word embeddings via deep compositional code learning.
\newblock {\em arXiv preprint arXiv:1711.01068}, 2017.

\bibitem[\protect\citeauthoryear{Shu and Nakayama}{2018}]{iclrShuN18}
Raphael Shu and Hideki Nakayama.
\newblock Compressing word embeddings via deep compositional code learning.
\newblock In {\em 6th International Conference on Learning Representations,
  {ICLR} 2018, Vancouver, BC, Canada, April 30 - May 3, 2018, Conference Track
  Proceedings}, 2018.

\bibitem[\protect\citeauthoryear{Sun \bgroup \em et al.\egroup
  }{2015}]{sun2015learning}
Fei Sun, Jiafeng Guo, Yanyan Lan, Jun Xu, and Xueqi Cheng.
\newblock Learning word representations by jointly modeling syntagmatic and
  paradigmatic relations.
\newblock In {\em Proceedings of the 53rd Annual Meeting of the Association for
  Computational Linguistics and the 7th International Joint Conference on
  Natural Language Processing (Volume 1: Long Papers)}, pages 136--145, 2015.

\bibitem[\protect\citeauthoryear{Tissier \bgroup \em et al.\egroup
  }{2019}]{tissier2019near}
Julien Tissier, Christophe Gravier, and Amaury Habrard.
\newblock Near-lossless binarization of word embeddings.
\newblock In {\em Proceedings of the AAAI Conference on Artificial
  Intelligence}, volume~33, pages 7104--7111, 2019.

\bibitem[\protect\citeauthoryear{Wang and Manning}{2012}]{wang2012baselines}
Sida Wang and Christopher~D Manning.
\newblock Baselines and bigrams: Simple, good sentiment and topic
  classification.
\newblock In {\em Proceedings of the 50th annual meeting of the association for
  computational linguistics: Short papers-volume 2}, pages 90--94. Association
  for Computational Linguistics, 2012.

\end{thebibliography}

\end{document}